\def\E{{\bf E}}
\def\S{K}
\def\reals{\mathbb{R}}

\def\cL{\mathcal{L}}

\def\hp{\hat{p}}

\documentclass{article}

\usepackage{graphicx} 
\usepackage{subfigure}
\usepackage{amsmath}
\usepackage{amsthm}
\newtheorem{theorem}{Theorem}
\newtheorem{lemma}{Lemma}

\usepackage{natbib}

\usepackage{algorithm}
\usepackage{amsfonts}
\usepackage{algorithmic}

\usepackage{hyperref}

 \usepackage[accepted]{icml2011}

\icmltitlerunning{Adaptively Learning the Crowd Kernel}

\begin{document}

\twocolumn[
\icmltitle{Adaptively Learning the Crowd Kernel}

\icmlauthor{Omer Tamuz}{omertamuz@weizmann.ac.il}
\icmladdress{Microsoft Research New England and Weizmann Institute of Science}
\icmlauthor{Ce Liu}{celiu@microsoft.com}
\icmladdress{Microsoft Research New England}
\icmlauthor{Serge Belongie}{sjb@cs.ucsd.edu}
\icmladdress{UC San Diego}
\icmlauthor{Ohad Shamir}{ohadsh@microsoft.com}
\icmlauthor{Adam Tauman Kalai}{adum@microsoft.com}
\icmladdress{Microsoft Research New England}

\icmlkeywords{active learning, crowdsourcing, kernels}

\vskip 0.3in
]

\begin{abstract}
  We introduce an algorithm that, given $n$ objects, learns a similarity matrix over all $n^2$
  pairs, from crowdsourced data {\em alone}.  The algorithm samples responses to
  adaptively chosen triplet-based relative-similarity queries.  Each query has the form
  ``is object $a$ more similar to $b$ or to $c$?'' and is chosen to be
  maximally informative given the preceding responses.  The output is
  an embedding of the objects into Euclidean space (like MDS); we refer to this
  as the ``crowd kernel.''  SVMs reveal that the crowd kernel captures prominent and subtle features across a number of domains, such as ``is striped'' among neckties and ``vowel vs.\ consonant'' among letters.

\end{abstract}

\section{Introduction}
Essential to the success of machine learning on a new domain is determining a good ``similarity function'' between objects (or alternatively defining good object ``features''). With such a ``kernel,'' one can perform a number of interesting tasks, e.g. binary classification using Support Vector Machines, clustering, interactive database search, or any of a number of other off-the-shelf kernelized applications.  Since this step of determining a kernel is most often the step that is still not routinized, effective systems for achieving this step are desirable as they hold the potential for completely removing the machine learning researcher from ``the loop.'' Such systems could allow practitioners with no machine learning expertise to employ learning on their domain. In many domains, people have a good sense of what similarity is, and in these cases the similarity function may be determined based upon crowdsourced human responses alone.

The problem of capturing and extrapolating a human notion of
perceptual similarity has received increasing attention in recent
years including areas such as vision \cite{Agarwal07}, audition \cite{McFee09}, information
retrieval \cite{Schultz03} and a variety of others represented in the
UCI Datasets \cite{Xing02,Huang10}.  Concretely, the goal of these
approaches is to estimate a similarity matrix $K$ over all pairs of
$n$ objects given a (potentially exhaustive) subset of human
perceptual measurements on tuples of objects.  In some cases the set
of human measurements represents `side information' to computed
descriptors (MFCC, SIFT, etc.), while in other cases -- the present
work included -- one proceeds exclusively with human reported data.
When $K$ is a positive semidefinite matrix induced purely from distributed human
measurements, we refer to it as the {\em crowd kernel} for the set of
objects.

Given such a Kernel, one can exploit it for a variety of purposes
including exploratory data analysis or embedding visualization (as in
Multidimensional Scaling) and relevance-feedback based interactive
search.  As discussed in the above works and \cite{Kendall90}, using a
{\em triplet based} representation of relative similarity, in which a
subject is asked ``is object $a$ more similar to $b$ or to $c$,'' has
a number of desirable properties over the classical approach employed
in Multi-Dimensional Scaling (MDS), i.e., asking for a numerical estimate of ``how similar is
object $a$ to $b$.''  These advantages include reducing fatigue on
human subjects and alleviating the need to reconcile individuals'
scales of similarity.  The obvious drawback with the triplet
based method, however, is the potential $O(n^3)$ complexity.  It is
therefore expedient to seek methods of obtaining high quality
approximations of $K$ from as small a subset of human measurements as
possible.   Accordingly, the primary contribution of this paper is an
efficient method for estimating $K$ via an information theoretic
adaptive sampling approach.

\begin{figure}
\center{\includegraphics[width=3in]{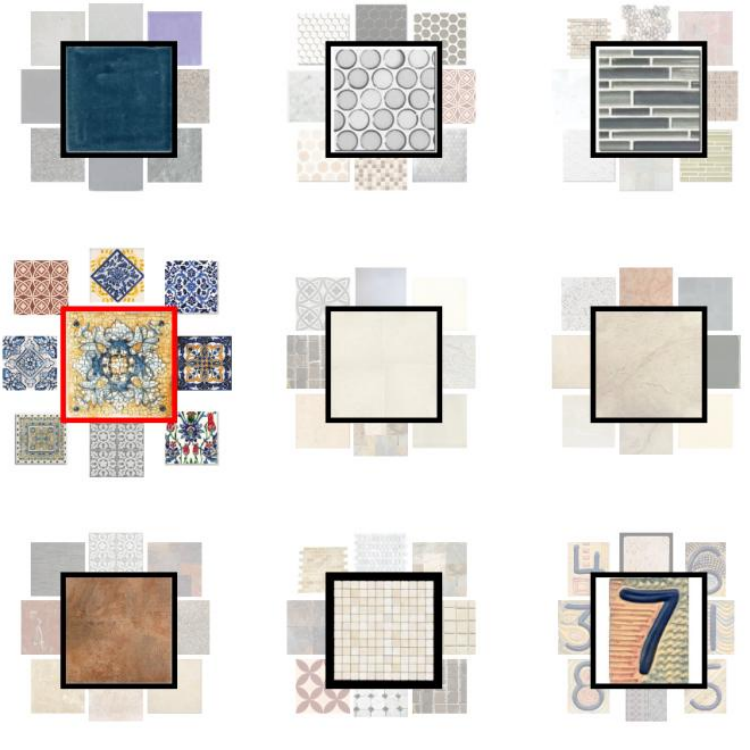}} \caption{\label{fig:tilestree} A sample top-level of a similarity search system that enables a user to search for objects by similarity.  In this case, since the user clicked on the middle-left tile, she will ``zoom-in'' and be presented with similar tiles.}
\end{figure}

At the heart of our approach is a new scale-invariant Kernel approximation model.  The choice of model is shown to be crucial in terms of the adaptive triples that are produced, and the new model produces effective triples to label.  Although this model is nonconvex, we prove that it can be optimized under certain assumptions.

We construct an {\em end-to-end} system for interactive visual search and
browsing using our Kernel acquisition algorithm. The input to this
system is a set of images of objects, such as products available in an
online store. The system automatically crowdsources\footnote{Crowdsourcing was done on Amazon's Mechanical Turk, http://mturk.com.} the kernel
acquisition and then uses this kernel to produce a visual interface
for searching or browsing the set of products. Figure
\ref{fig:tilestree} shows this interface for a dataset of 433 floor
tiles available at amazon.com.

\subsection{Human kernels versus machine kernels}
The bulk of work in Machine Learning focuses on ``Machine Kernels''
that are computed by computer from the raw data (e.g., pixels)
themselves.  Additional work employs human experiments to try to learn
kernels based upon machine features, i.e., to approximate the human
similarity assessments based upon features that can be derived by
machine.  In contrast, when a kernel is learned from human subjects
alone (whether it be data from an individual or a crowd) one requires
no machine features whatsoever. To the computer, the objects
are recognized by ID's only -- the images themselves are hidden from
our system and are only presented to humans.

The primary advantage of machine kernels is that they can {\em generalize} immediately to new data, whereas each additional object needs to be added to our system, for a cost of approximately \$0.15.\footnote{This price was empirically observed to yield ``good performance'' across a number of domains.  See the experimental results section for evaluation criteria.} On the other hand, working with a human kernel has two primary advantages.  First, it does not require any domain expertise.  While for any particular domain, such as music or images of faces, cars, or sofas, decades of research may have provided high-quality features, one does not have to find, implement, and tune these sophisticated feature detectors.

Second, human kernels contain features that are simply not available with state-of-the-art feature detectors, because of knowledge and experience that humans possess.  For example, from images of celebrities, human similarity may be partly based on whether the two celebrities are both from the same profession, such as politicians, actors, and so forth.  Until the longstanding goal of bridging the semantic gap is achieved, humans will be far better than machines at interpreting certain features, such as ``does a couch look comfortable,'' ``can a shoe be worn to an informal occasion,'' or ``is a joke funny.''

We give a simple demonstration of external knowledge through experiments on 26 images of the lower-case Roman alphabet.  Here, the
learned Kernel is shown to capture features such as ``is a letter a vowel versus consonant,'' which uses
external knowledge beyond the pixels.  Note that this experiment is interesting in itself because it is not at first clear if people can
meaningfully answer the question: ``is the letter {\em e} more similar to {\em i} or {\em p}.'' Our experiments show statistically significant consistency with 58\%\footnote{While this fraction of agreement seems small, it corresponds to about 25\% ``noise,'' e.g., if 75\% of people would say that $a$ is more like $b$ then $c$, then two random people would agree with probability $0.75^2=0.56$.}($\pm$2\%, with 95\% confidence) agreement between users on a random
triple of letters.  (For random image triples from an online tie
store, 68\% agreement is observed, and 65\% is observed for floor tile images).

\section{Benefits of adaptation}
We first give high-level intuition for why adaptively choosing triples may
yield better kernel approximations than randomly chosen triples.  Consider $n$ objects organized in a rooted tree with $\ell \ll n$ leaves, inspired by, say, phylogenic trees involving animal species.\footnote{This example is based upon a tree metric rather than a Euclidean one.  However, note that any tree with $\ell$ leaves can be embedded in $\ell$-dimensional Euclidean space, where squared Euclidean distance equals tree distance.}  Say the similarity between objects is decreasing in their distance in the tree graph and, furthermore, that objects are drawn uniformly at random from the classes represented by the leaves of the tree.  Ignoring the details of how one would identify that two objects are in the same leaf or subtree, it is clear that a nonadaptive method would have to ask $\Omega(n \ell)$ questions to determine the leaves to which $n$ objects belong (or at least to determine which objects are in the same leaves), because an expected $\Omega(\ell)$ queries are required per object until just a second object is chosen from the same leaf.  On the other hand, in an ideal setting, an adaptive approach might determine such matters using $O(n \log \ell)$ queries in a balanced binary tree, proceeding from the root down, assuming a constant number of comparisons can determine to which subtree of a node an object belongs, hence an exponential savings.

\section{Related work}

As discussed above, much of the work in machine learning on learning kernels employs `side information' in the form of features about objects.  \citet{Schultz03} highlight the fact that triple-based information may also be gathered by web search click data.
\citet{Agarwal07} is probably the most similar work, in which they learn a kernel matrix from triples of similarity comparisons, as we do.  However, the triples they consider are randomly (nonadaptively) chosen.  Their particular fitting algorithm differs in that it is based on a max-margin approach, which is more common in the kernel learning literature.

There is a wealth of work in {\em active learning} for classification, where a learner selects examples from a pool of unlabeled examples to label.  A number of approaches have been employed, and our work is in the same spirit as those that employ probabilistic models and information-theoretic measures to maximize information.  Other work often labels examples based on those that are closest to the margin or closest to 50\% probability of being positive or negative.  To see why this latter approach may be problematic in our setting, one could imagine a set of triples where we have accurately learned that the answer is $50/50$, e.g., as may be the case if $a$, $b$, and $c$ bear no relation to each other or if they are identical.  One may not want to focus on such triples.

\section{Preliminaries}\label{sec:prelim}

The set of $n$ objects is denoted by $[n]=\{1,2,\ldots,n\}$.  For $a,b,c \in [n]$, a comparison or {\em triple} is of the form, ``is $a$ more similar to $b$ or to $c$.'' We refer to $a$ as the {\em head} of the triple.  We write $p^a_{bc}$ for the probability that a {\em random} crowd member rates $a$ as more similar to $b$, so $p^a_{bc}+p^a_{cb}=1$.
The $n$ objects are assumed to have $d$-dimensional Euclidean representation, and hence the data can be viewed as a matrix $M \in \reals^{n \times d}$, where $M_a$ denotes the row corresponding to $a$, and the {\em similarity matrix} $\S \in \reals^{n \times n}$ is defined by $\S_{ab}=M_a\cdot M_b$, or equivalently $\S = MM^T$.  The goal is to learn $M$ or, equivalently, learn $K$.  (It is easy to go back and forth between positive semidefinite (PSD) $K$ and $M$, though $M$ is only unique up to change of basis.)  Also equivalent is the representation in terms of distances, $d^2(a,b)=\S_{aa}-2\S_{ab}+\S_{bb}$.

In our setting, an {\em MDS algorithm} takes as input $m$ comparisons
$(a_1b_1c_1,y_1) \ldots (a_mb_mc_m,y_m)$ on $n$ items, where $y_i\in
\{0,1\}$ indicates whether $a_i$ is more like $b_i$ than $c_i$.
Unless explicitly stated, we will often omit $y_i$ and assume that the
$b_i$ and $c_i$ have been permuted, if necessary, so that $a_i$ was
rated as more similar to $b_i$ than $c_i$.  The MDS algorithm outputs
an embedding $M \in \reals^{n \times d}$ for some $d \geq 1.$ A
probabilistic MDS model predicts $\hat{p}^{a}_{bc}$ based on $M_a$,
$M_b$, and $M_c$.  The {\em empirical log-loss} of a model that
predicts $\hp^{a_i}_{b_ic_i}$ is $1/m\sum_i \log 1/\hp^{a_i}_{b_ic_i}$.
Our probabilistic MDS model attempts to minimize empirical log loss
subject to some regularization constraint.  We choose a probabilistic
model due to its suitability for use in combination with our
information-gain criteria for selecting adaptive triples, and also due
to the fact that the same triple may elicit different answers from
different people (or the same person on different occasions).

An {\em active} MDS algorithm chooses each triple, $a_ib_ic_i$, adaptively based upon previous labels $(a_1b_1c_1,y_1)$,$\ldots$,$(a_{i-1}b_{i-1}c_{i-1},y_{i-1})$.
We denote by $M^T$ the transpose of matrix $M$.
For compact convex set $W$, let $\Pi_W(\S)=\arg\min_{T \in W} \sum_{ij}(\S_{ij}-T_{ij})^2$ be the closest matrix in $W$ to $\S$.  Also define the set of symmetric  unit-length PSD matrices,
$$B=\{ \S \succeq 0 ~|~ \S_{11}=\S_{22}=\ldots=\S_{nn}=1\}.$$
Projection to the closest element of $B$ is a quadratic program which can be solved via a number of existing techniques \cite{SS05,LRSST10}.

\section{Our algorithm}

Our algorithm proceeds in phases.  In the first phase, it queries a certain number of random triples comparing each object $a \in [n]$ to random pairs of distinct $b,c$.  (Note that we never present a triple where $a=b$ or $a=c$ except for quality control purposes.) Subsequently, it fits the results to a matrix $M \in \reals^{n \times d}$ (equivalently, fits $K\succeq 0$) using the probabilistic {\em relative} similarity model described below.  Then it uses our adaptive selection algorithm to select further random triples.  This iterates: in each phase all previous data is refit to the relative model, and then the adaptive selection algorithm generates more triples.
\begin{itemize}
\item For each item $a\in [n]$, crowdsource labels for $R$ random triples with head $a$.
\item For $t=1,2,\ldots,T:$
\begin{itemize}
\item Fit $\S^t$ to the labeled data gathered thus far, using the method described in Section \ref{sec:rel} (with $d$ dimensions).
\item For each $a\in [n]$, crowdsource a label for the maximally informative triple with head $a$, using the method described in Section \ref{sec:ada}.
\end{itemize}
\end{itemize}
Typical parameter values which worked quickly and well across a number of medium-sized data sets of (hundreds of objects) were $R=10$, $T=25$, and $d=3$.  These settings were also used to generate Figure \ref{fig:YAY}.  We first describe the probabilistic MDS model and then the adaptive selection procedure.  Further details are given in Section \ref{sec:params}.

\subsection{Relative similarity model}\label{sec:rel}

The {\em relative} similarity model is motivated by the scale-invariance observed in many perceptual systems (see, e.g., \citealp{CB99}).  Let $\delta_{ab} = \|M_a-M_b\|^2=\S_{aa}+\S_{bb}-2\S_{ab}$.  A simple scale-invariant proposal takes $\hat{p}^a_{bc} = \frac{\delta_{ac}}{\delta_{ab}+\delta_{ac}}$.  Such a model must also be regularized or else it would have $\Theta(n^2)$ degrees of freedom.  One may regularize by the rank of $\S$ or by setting $\S_{ii}=1$.  Due to the scale-invariance of the model, however, this latter constraint does not have reduced complexity.  In particular, note that halving or doubling the matrix $M$ doesn't change any probabilities.  Hence, descent algorithms may lead to very small, large, or numerically unstable solutions.  To address this, we modify the model as follows, for distinct $a,b,c$:
\begin{equation}
\label{eq:rel}
\hat{p}^a_{bc} =  \frac{\mu+\delta_{ac}}{2\mu+\delta_{ab}+\delta_{ac}} \ \textrm{ and }\  \S_{ii}=1,
\end{equation}
for some parameter $\mu>0$.  Alternatively, this change may be viewed as an additional assumption imposed on the previous model --  we suppose each object possesses a minimal amount of ``uniqueness,'' $\mu>0$, such that $\S = \mu I + T$, where $T \succeq 0$.  We fit the model by local optimization performed directly on $M$ (with random initialization), and high-quality adaptive triples are produced even for low dimensions.\footnote{For high-dimensional problems, we perform a gradient projection descent on $\S$.  In particular,
starting with $\S^0=I$, we compute $\S^{t+1}= \Pi_B(\S^t - \eta \nabla \cL(\S))$ for step-size $\eta$ (see Preliminaries for the definition of $\Pi_B$).}  Here $\mu$ serves a purpose similar to a margin constraint.

There are two interesting points to make about our choice of model.  First, the loss is not convex in $\S$, so there is a concern that local optimization may be susceptible to local minima.  In Section \ref{sec:theory}, we state a theorem which explains why this does not seem to be a significant problem.  Second, in Section \ref{sec:logistic}, we discuss a simple convex alternative based on logistic regression, and we explain why this model, in combination with our adaptive selection criterion, gives rise to poor adaptively-selected triples.

\section{Adaptive selection algorithm}\label{sec:ada}

The idea is to capture the uncertainty about the
location of an object through a probability distribution over points
in $\reals^d$, and then to ask the question that maximizes information
gain. Given a set of previous comparisons of $n$ objects, we generate, for
each object $a=1,2,\ldots,n$, a new triple to compare $a$ to, as
follows.  First, we embed the objects into $\reals^d$ as described
above, using the available comparisons. Initially, we use a seed of
randomly selected triples for this purpose. Later, we use all
available comparisons - the initial random ones and those acquired
adaptively.

Now, say the crowd has previously rated $a$ as more similar to $b_i$
than $c_i$, for $i=1,2,\ldots,j-1$, and we want to generate the $j$th
query, $(a,{b_j,c_j})$ (this is a slight abuse of notation because we
don't know which of $b_j$ or $c_j$ will be rated as closer to
$a$). These observations imply a posterior distribution of $\tau(x)
\propto \pi(x) \prod_i \hp^{x}_{b_ic_i}$ over $x \in \reals^d$, where
$x$ is the embedding of $a$, and $\pi(x)$ is a prior distribution, to
be described shortly.

Given any candidate query for objects in the database $b$ and $c$, the
model predicts that the crowd will rate $a$ as more similar to $b$
than $c$ with probability $p \propto \int_x
\frac{\delta(x,c)}{\delta(x,b)+\delta(x,c)}\tau(x)dx$.  If it rates $a$ more similar to $b$ than $c$ then
$x$ has a posterior distribution of $\tau_b(x) \propto
\tau(x)\frac{\delta(x,c)}{\delta(x,b)+\delta(x,c)}$, and $\tau_c(x)$
(of similar form) otherwise.  The {\em information gain} of this query
is defined to be $H(\tau)-pH(\tau_b)-(1-p)H(\tau_a)$, where $H(\cdot)$
is the entropy of a distribution. This is equal to the mutual
information between the crowd's selection and $x$. The algorithm
greedily selects a query, among all pairs $b,c \neq a$, which
maximizes information gain.  This can be somewhat
computationally intensive (seconds per object in our datasets), so for
efficiency we take the best pair from a sample of random pairs.

It remains to explain how we generate the prior $\pi$.  We take $\pi$
to be the uniform distribution over the set of points in $M$.  Hence,
the process can be viewed as follows.  For the purpose of generating a
new triple, we pretend the coordinates of all other objects are
perfectly known, and we pretend that the object in question, $a$, is
a uniformly random one of these other objects.  The chosen pair is designed to
maximize the information we receive about which object it is, given
the observations we already have about $a$.  The hope is that, for
sufficiently large data sets, such a data-driven prior is a reasonable
approximation to the actual distribution over data.  Another natural
alternative prior would be a multinormal distribution fit to $M$.

\subsection{Optimization guarantee}\label{sec:theory}

The relative similarity model is appealing in that it fits the data well, suggests good triples, and also represents interesting features on the data.  Unfortunately, the model itself is not convex.  We now give some justification for why gradient descent should not get trapped in local minima.  As is sometimes the case in learning, it is easier to analyze an online version of the algorithm, i.e., a stochastic gradient descent.  Here, we suppose that the sequence of triples is presented in order: the learner predicts $\S^{t+1}$ based on $(a_1,b_1,c_1,y_1),\ldots,(a_t,b_t,c_t,y_t)$.  The loss on iteration $t$ is $\ell_t(\S^t)=\log 1/p$ where $p$ is the probability that the relative model with $\S^t$ assigned to the correct outcome.

We state the following theorem about stochastic gradient descent,
where $\S^0 \in B$ is arbitrary and $\S^{t+1}=\Pi_B(\S^t-\eta
\nabla\ell_t(\S^t))$. 

\begin{theorem}
  \label{thm:one}
Let $a_t,b_t,c_t \in [n]$ be arbitrary, for $t=1,2,\ldots$.  Suppose there is a matrix $\S^*\in B$ such that $\Pr[y_t=1]=\frac{\mu+2-2\S^*_{ac}}{2\mu+4-2\S^*_{ab}-2\S^*_{ac}}$.  For any $\epsilon>0$, there exists an $T_0$ such that for any $T>T_0$ and $\eta=1/\sqrt{T}$,
$${\mathrm E}\left[\frac{1}{T}\sum_{t=1}^T \ell_t(\S^t)-\ell_t(\S^*) \right]\leq \epsilon.$$
\end{theorem}
We prove this theorem in Appendix~\ref{app:proof}.
\subsection{The logistic model: A convex alternative}\label{sec:logistic}
As a small digression, we explain why the choice of probabilistic model is especially important for adaptive learning.  To this end, consider the following {\em logistic} model.  This model is a natural hybrid of logistic regression and MDS.
\begin{equation}\label{eq:logistic}
\hat{p}^a_{bc} = \frac{e^{\S_{ab}}}{e^{\S_{ab}}+e^{\S_{ac}}} = \frac{1}{1+e^{\S_{ac}-\S_{ab}}}.
\end{equation}
Note that $\log 1+e^{\S_{ac}-\S_{ab}}$ is a convex function of $\S\in \reals^{n\times n}$.  Hence, the problem of minimizing its empirical log loss over a convex set is a convex optimization problem.

Experiments indicate that the logistic model fits data well and reproduces interesting features, such as vowel/consonant or stripedness.  However, empirically it performs poorly in terms of {\em deciding which triples to ask}.  Note that the logistic model (and any generalized linear model) will select extreme comparisons, where the inner products are as large as possible.  To give an intuitive understanding, suppose that hair length was a single feature and one wanted to determine whether a person has hair length $x$ or $x+1$.  The logistic model would compare that person to a bald person ($x=0$) and a person with hair length $2x+1$, while the relative model would ideally compare him to people with hair lengths $x$ and $x+1$.

\section{System parameters \& quality control}\label{sec:params}
Experiments were performed using Amazon's Mechanical Turk web service,
where we defined `Human Intelligence Tasks' to be performed by one or
more users.  Each task consists of 50 comparisons and the interface is
optimized to be performed with 50 mouse clicks (and no scrolling).
The mean completion time was approximately 2 minutes.  This payment
was determined based upon worker feedback.  Initial experiments
revealed a high percentage of seemingly random responses, but after
closer inspection the vast majority of these poor results came from a
small number of individuals.  To improve quality control, we imposed a
limit on the maximum number of tasks a single user could perform on
any one day, we selected users who had completed at least 48 tasks
with a 95\% approval rate, and each task included 20\% triples for
which there was tremendous agreement between users.  These ``gold
standard'' triples were also automatically generated and proved to be
an effective manner to recognize and significantly reduce cheating.
The system is implemented using Python, Matlab, and C, and runs
completely automatically in Windows or Unix.

\subsection{Question phrasing and crowd alignment}
One interesting issue is how to frame similarity questions.  On the
one hand, it seems purest in form to give the users carte blanche and
ask only, ``is $a$ more similar to $b$ than $c$.''  On the other hand,
in feedback users complained about these tasks and often asked what we
meant by similarity.  Moreover, different users will inevitably weigh
different features differently when performing comparisons.

For example, consider a comparisons of face images, where $a$ is a
white male, $b$ is a black male, and $c$ is a white female.  Some
users will consider gender more important in determining skin color,
and others may feel the opposite is true.  Others may feel that the
question is impossible to answer.  Consider phrasing the question as
follows, ``At a {\em distance}, who would you be more likely to
mistake for $a$: $b$ or $c$?''  For any two people, there is
presumably some distance at which one might be mistaken for the other,
so the question may seem more possible to answer for some people.
Second, users may more often agree that skin color is more important
than gender, because both are easily identified close up by skin color
may be identifiable even at a great distance.  While we haven't done
experiments to determine the importance of question phrasing,
anecdotal evidence suggests that users enjoy the tasks more when more
specific definitions of similarity are given.

Two natural goals of question phrasing might be: (1) to align users in
their ranking of the importance of different features and (2) to align
user similarity notions with the goals of the task at hand.  For
example, if the task is to find a certain person, the question,
``which two people are most likely to be (genealogically) related to
one another,'' may be poor because users may overlook features such as
gender and age.  In our experiments on neckties, for example, the task
was titled ``Which ties are most similar?'' and the complete
instructions were: {\em ``Someone went shopping at a tie store and
  wanted to buy the item on top, but it was not available. Click on
  item (a) or (b) below that would be the {\bf best substitute}.''}

\section{Experiments and Applications}


\begin{figure}
\center{\includegraphics[width=3.2in]{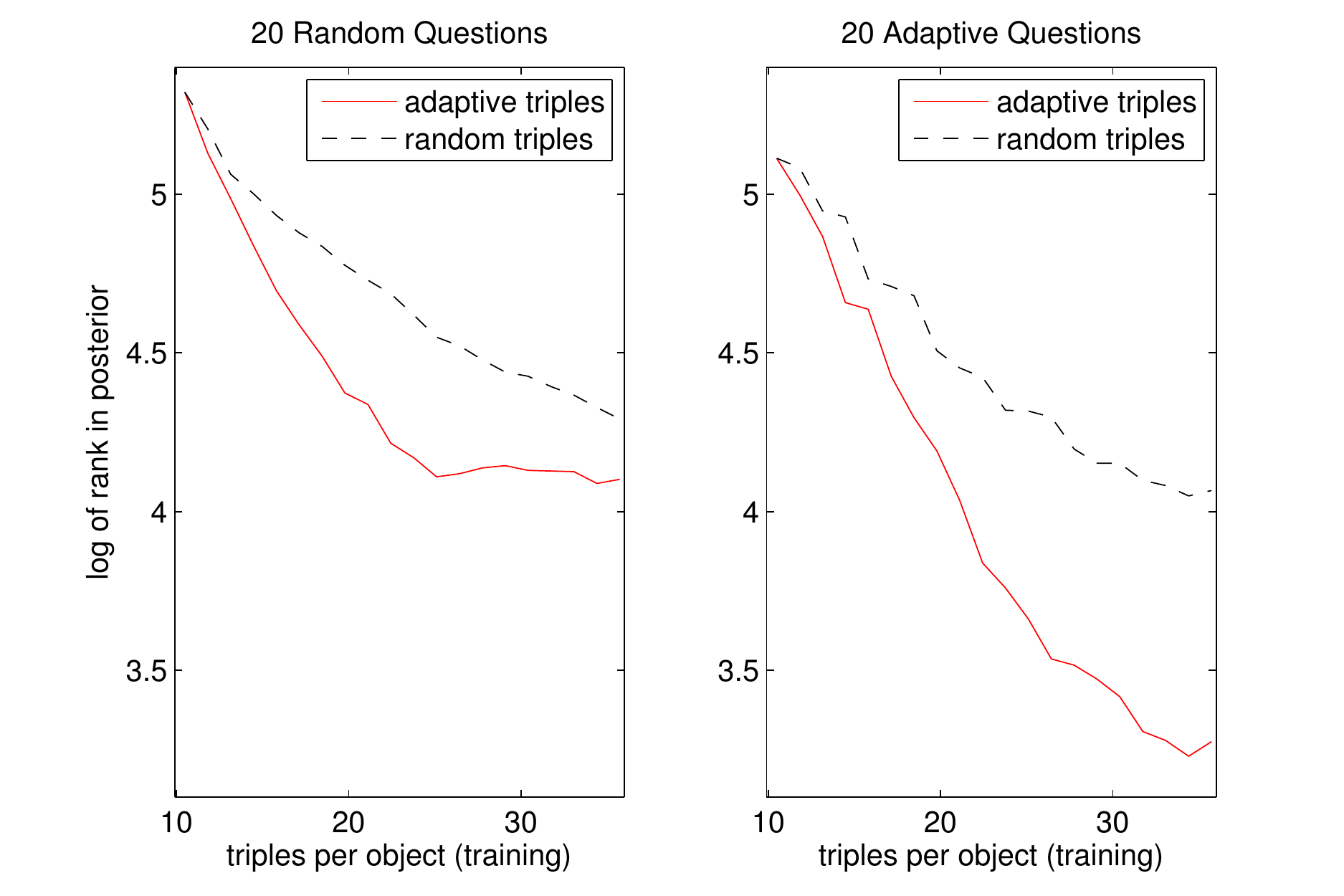}} \caption{\label{fig:YAY} The 20Q plots comparing training based on adaptively selected triples to randomly selected training triples.  The left plot shows the mean predicted log-ranks of randomly chosen objects after 20 randomly chosen questions. The right plot shows the mean predicted log-ranks of randomly chosen objects after 20 adaptive queries.
Plots were generated using the mixed dataset consisting of $n=225$ objects, with 10 initial random triples per object.  In both plots, the performance using $22 = (10 $ random$) + (12$ adaptively chosen$)$ triples was matched using all 35 random triples.  Hence, approximately 60\% more random triples were required to match this particular performance level of the adaptive algorithm.
}
\end{figure}

We experiment on four datasets: (1) twenty-six images of the lowercase
roman alphabet (Calibri font) (2) 223 country flag images from
flagpedia.net, (3) 433 floor tile images from Amazon.com, and (4) 300
product images from an online tie store also hosted at Amazon.com. We
also consider a hand-selected ``mixed'' dataset consisting of 225
images: 75 ties, 75 tiles, and 75 flags.  Surprisingly, it seems that
for these datasets about 30-40 triples per object suffice to learn the
Crowd Kernel well, according to the 20Q metric that we describe
below.. Figure \ref{fig:YAY} shows the results on the mixed dataset,
comparing the 20Q metric trained on random vs.\ adaptive triples.  For
both adaptive and random questions, for certain performance levels,
one requires about 60\% more random queries than adaptive queries.
Given very little data or a lot of data, one does not expect the
adaptive algorithm to perform significantly better.

Figure \ref{fig:adaptive-trips} shows the adaptive triples selected on
an illustrative dataset composed of a mixture of flags, ties and
tiles.

For ease of implementation, we assume all users are identical.  This
is a natural starting point, especially given that our main focus is
on active learning.

\begin{figure}
\center{\includegraphics[width=3.4in]{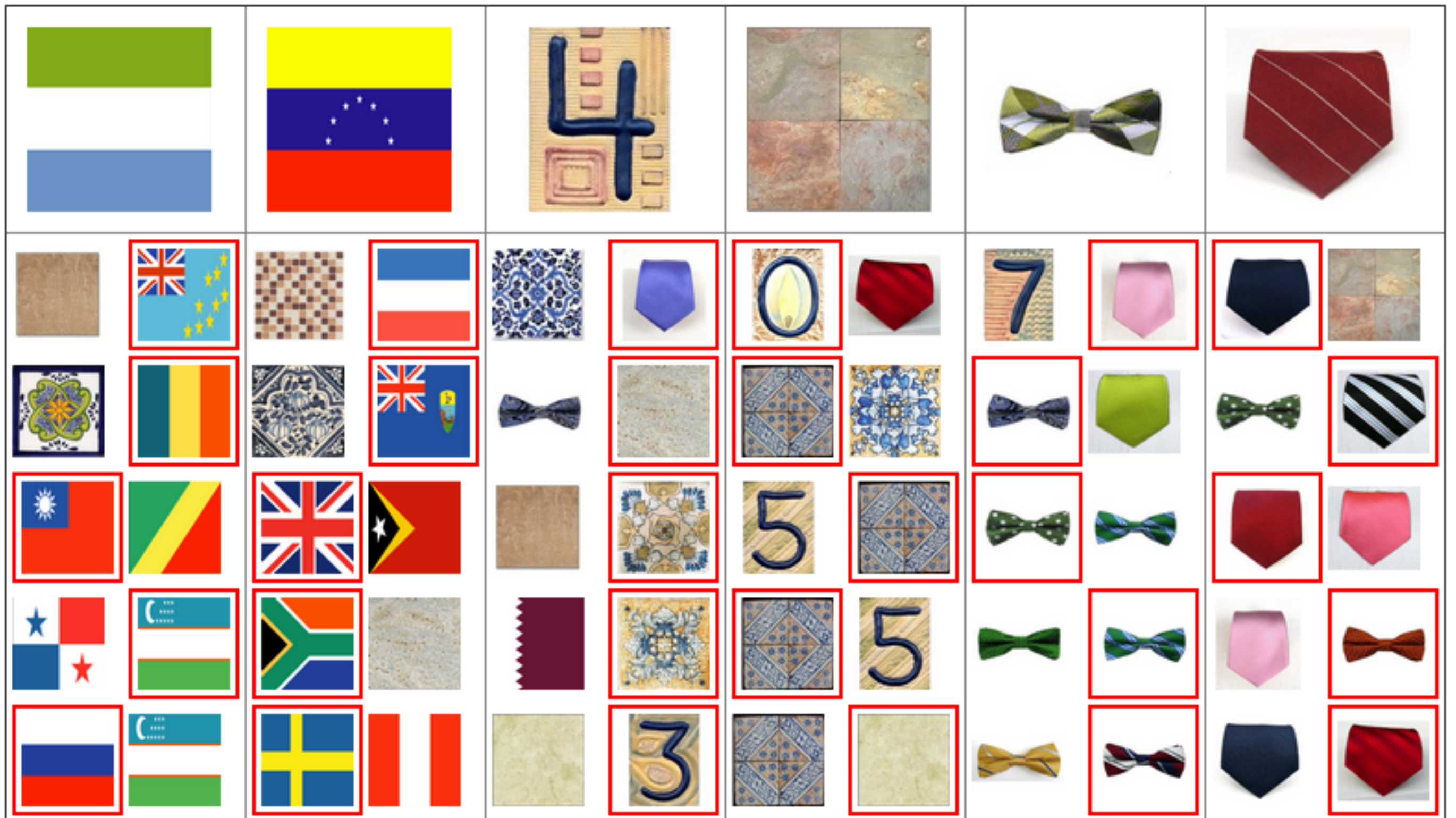}} \caption{\label{fig:adaptive-trips} Six objects in the mixed dataset along with the adaptive pairs to which that object was compared, below, and user selections in red.  The first pair below each large object was chosen adaptively based upon the results of ten random comparisons.  Then, proceeding down, the pairs were chosen using the ten random comparisons plus the results of the earlier comparisons above.}
\end{figure}
\begin{figure}
{\center \includegraphics[width=3.4in]{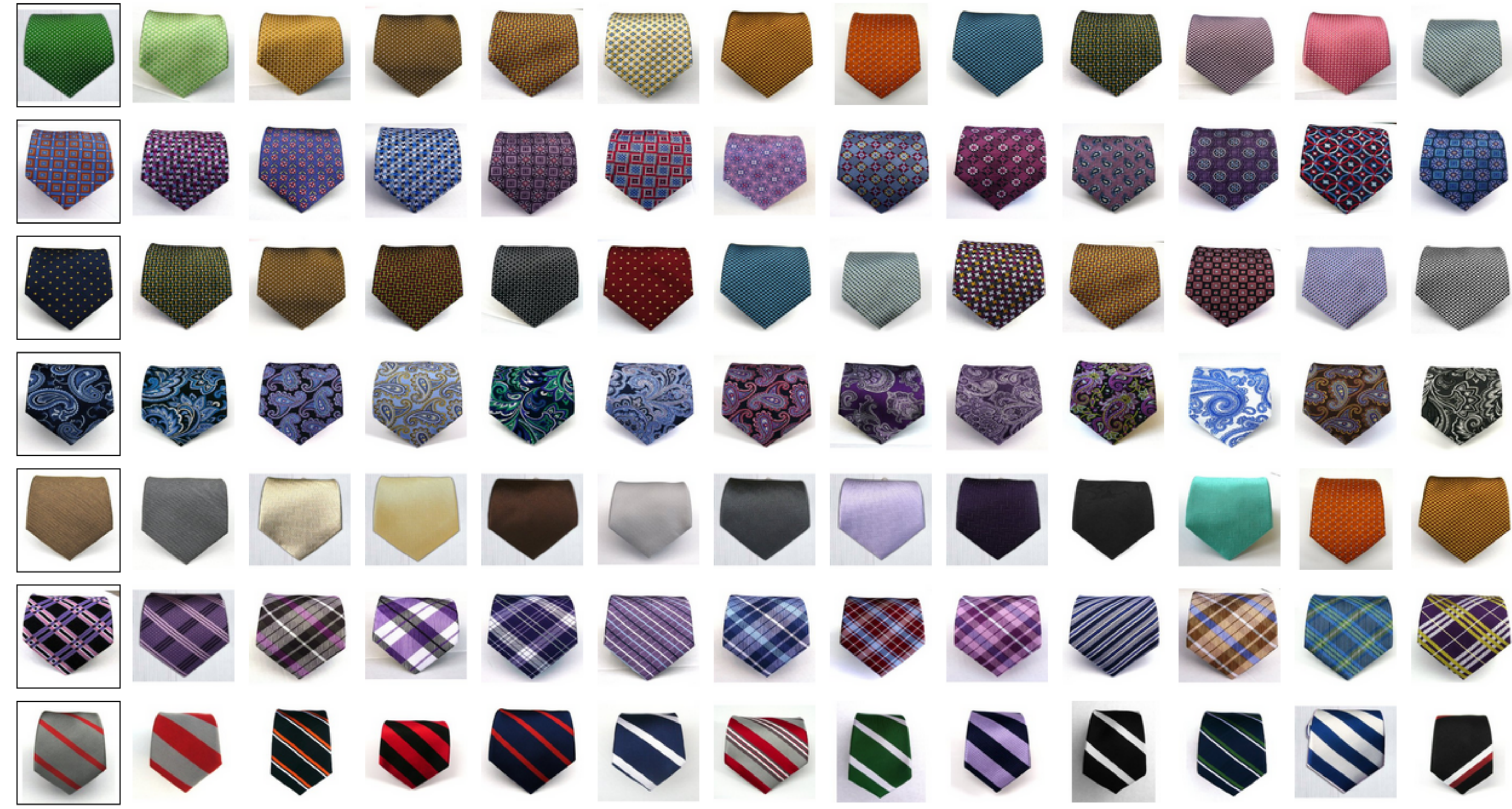}}
\caption{\label{fig:nn} Nearest-neighbors for some neckties from the tie-store dataset.  Nearest neighbors are displayed from left to right.  Note that neck ties were never confused for bow ties, tie clips or scarves.
}
\end{figure}
\begin{figure}
{\center \includegraphics[width=3.4in]{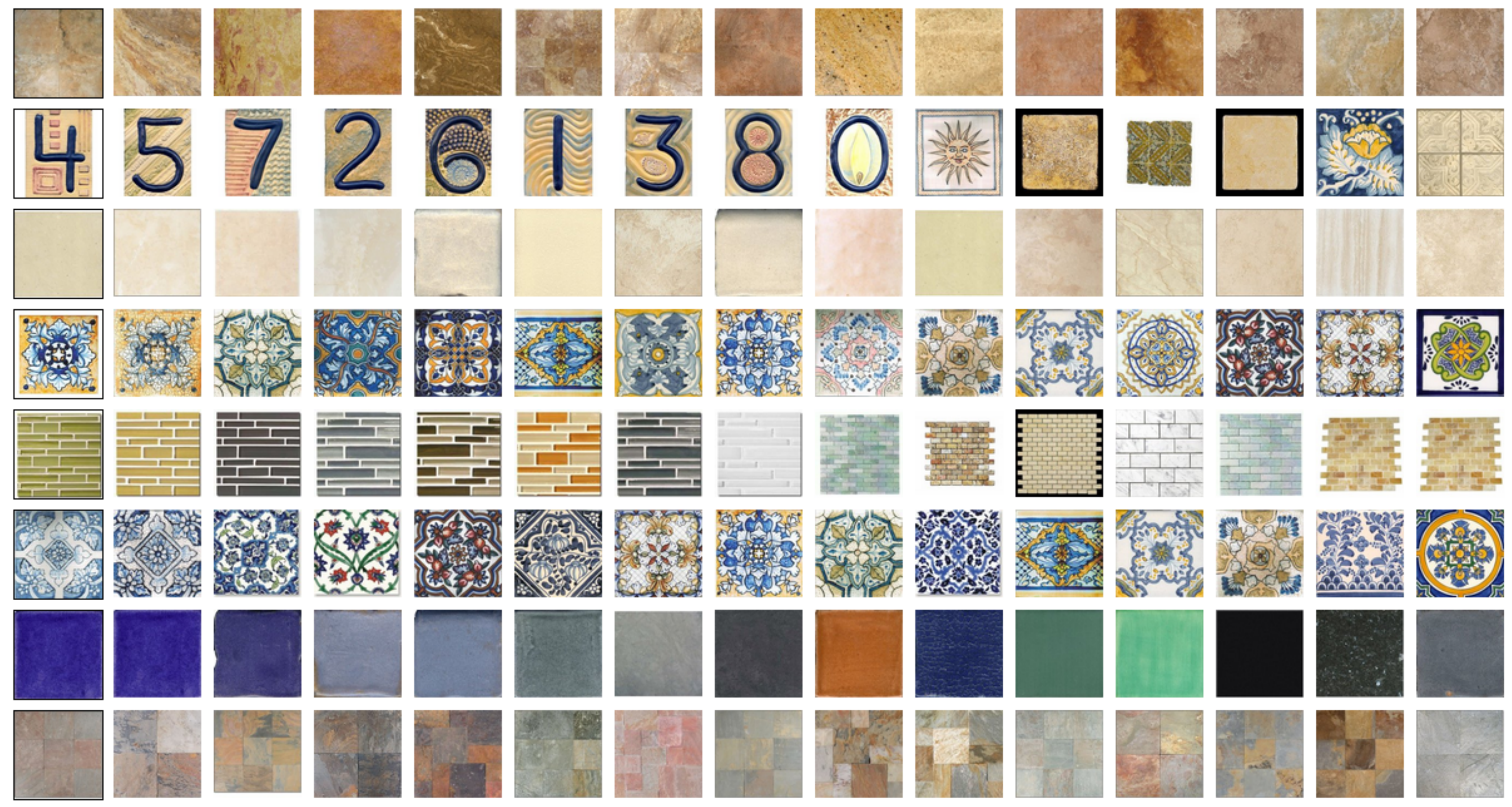}} \caption{\label{fig:nn-tiles} Examples of nearest neighbors for floor tiles.}
\end{figure}
\begin{figure}
{\center \includegraphics[width=3.5in]{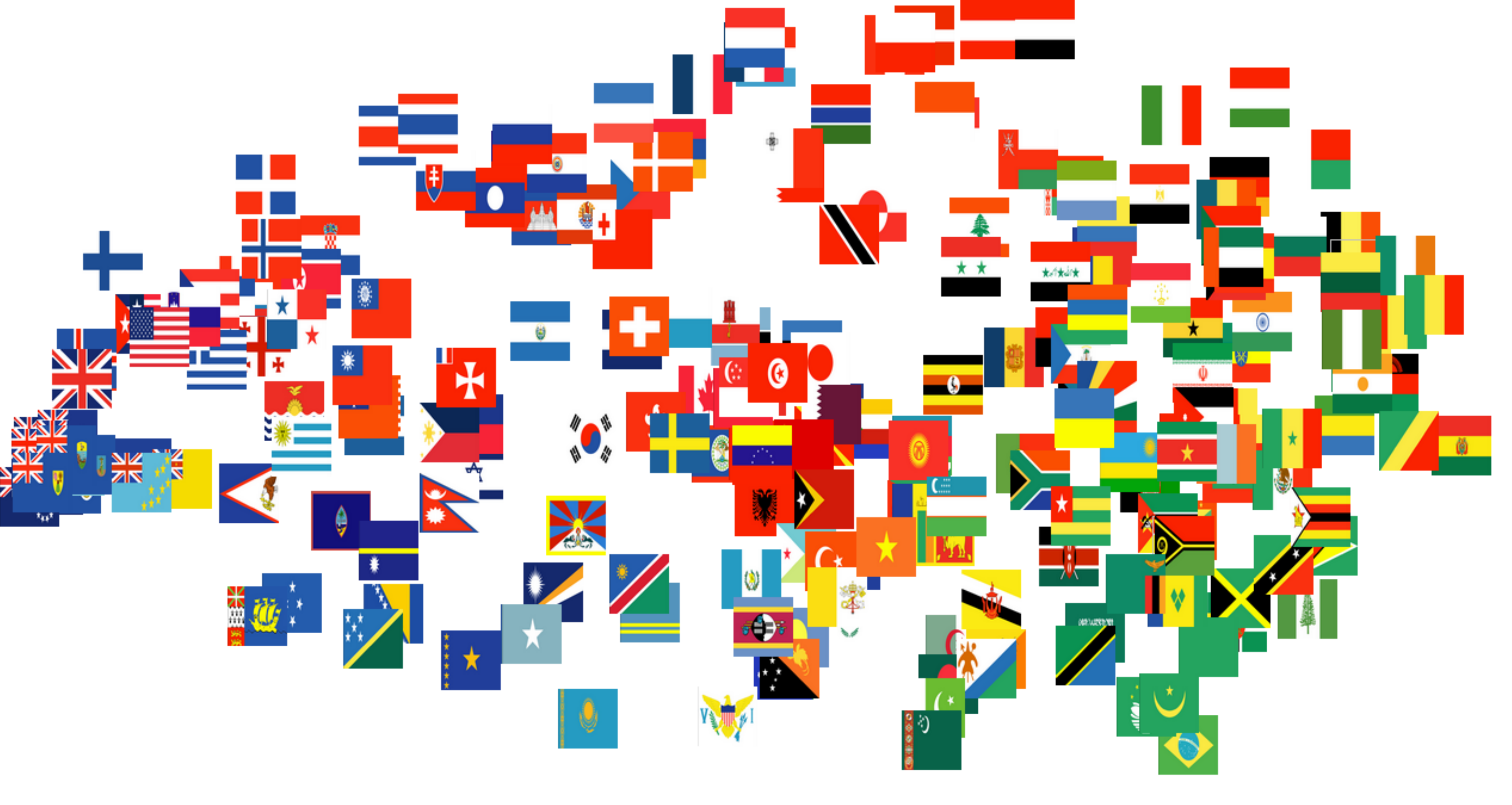}} \caption{\label{fig:flagspca} The flag images displayed according to their projection on the top two principal components of a PCA.  The principal component
is the horizontal axis.}
\end{figure}
\begin{figure}
{\center \includegraphics[width=3.5in]{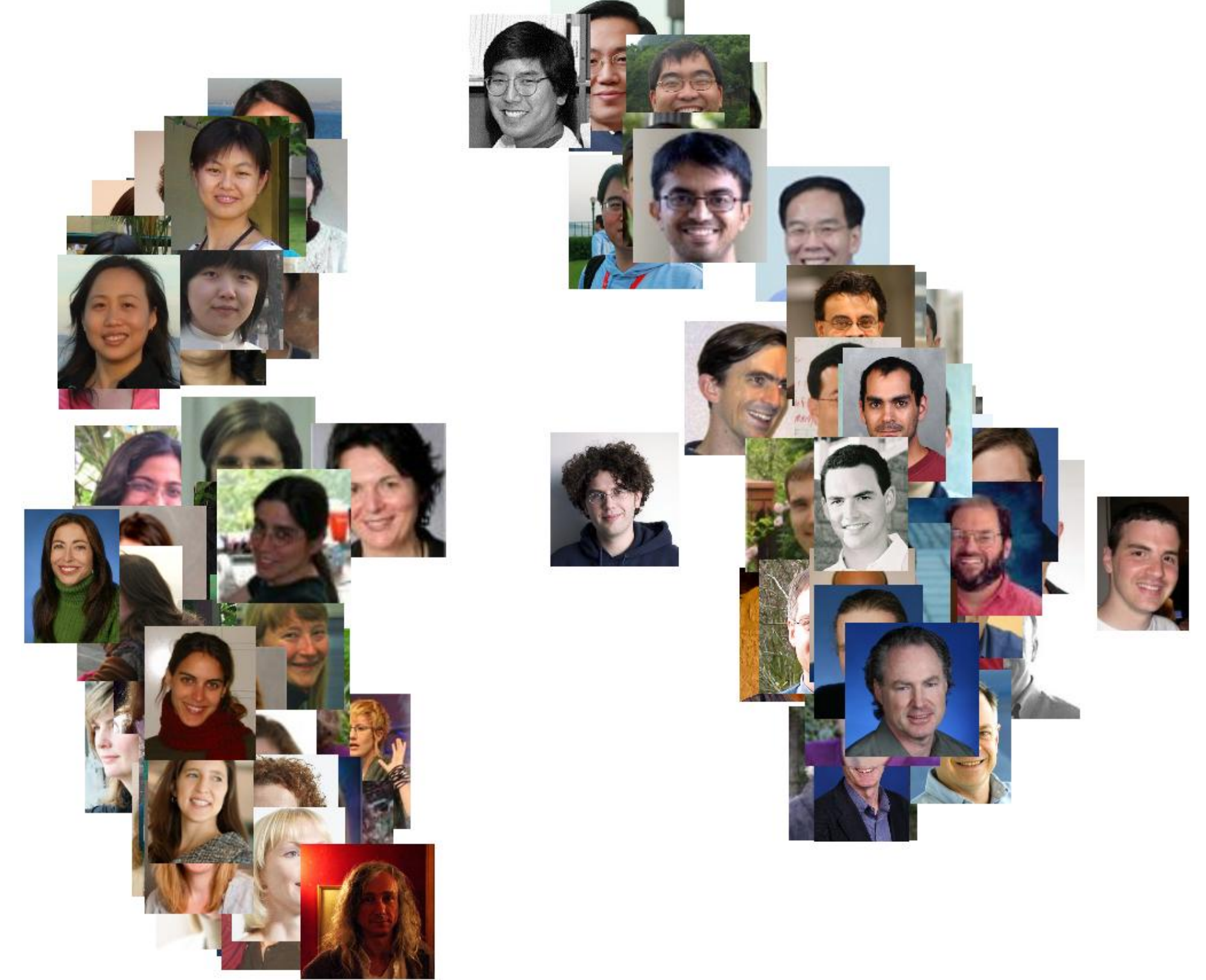}} \caption{\label{fig:facespca} For fun: the faces of 186 colleagues displayed according to their projection on the top two principal components. The principal component
is the horizontal axis.}
\end{figure}

\begin{figure}
{\center
\begin{tabular}{|l|l|l|}
 \hline
 Dataset & Feature & LOO error rate \\
 \hline
 \hline
Tiles & Ornate & 4.1\% \\
Ties & Bow tie vs.~neck tie & 0.0\%\\
Ties & Multicolor vs. plain & 0.5\%\\
Flags & Striped & 0.0\%\\
Letters & Vowel & 4.0\%\\
Letters & Short/tall & 5.3\%\\
\hline
\end{tabular}
\caption{Empirical results of an SVM using the crowd kernel, based on leave-one-out cross validation. Note that in many cases we hand-selected ``easy'' subsets of objects to label.  For example, when judging whether a flag is striped or not, we removed flags which might be interpreted either way.  The selection was based on how unambiguous the objects were, with respect to the desired label, and not related to the target kernel.
\label{fig:results}}
}\end{figure}

\subsection{20 Questions Metric}
Since one application of such systems is search, i.e., searching for an item that a user knows what it looks like (we assume that the user can answer queries as if she even knows what the store image looks like), it is natural to ask how well we have ``honed in'' on the desired object after a certain number of questions.  For the 20 Questions (20Q) metric, two independent parts of the system are employed, an {\em evaluator} and a {\em guesser}.  First, the evaluator randomly and secretly selects an object in the database, $x$.  The guesser is aware of the database but not of which item $x$ has been selected.  The guesser is allowed to query 20 triples (as in the game ``20 Questions'') with head $x$, after which it produces a ranking of items in the database, from most to least likely.  Then the evaluator reveals the identity of $x$ and hence its position in the ordered ranking, as well.  The metric is the average $\log$ of the position of the random target item in this list.  The $\log$  reflects the idea that the position of lower-ranked objects is less import -- it weights moving an object from position 2 to 4 as important as moving an object from position 20 to 40.  This metric is meant to roughly capture performance, but of course in a real system users may not have the patience to click on twenty pairs of images and may prefer to have fewer clicks but use larger comparison sets. (Our GUI has the user select one of 8 or 9 images, which could potentially convey the same information as 3 binary choices.)  Now, the questions that the guesser asks could be random questions, which we refer to as the 20 Random Questions metric, or adaptively chosen, for the 20 Adaptive Questions metric.  In the latter case, the guesser uses the same maximum information-gain criterion as in the adaptive triple generation algorithm, relative to whichever model was learned (based on random or adaptively selected training triples).

\subsection{Using the Kernel for Classification}

The learned Kernels may be used in a linear classifier such as a
support vector machine.  This helps elucidate which features have been used
by humans in labeling the data.  In the experiments below, an unambiguous subset of the images were
labeled with binary $\pm$ classes. For example, we omitted the letter {\em y} in labeling vowels and consonants ({\em y} was in fact classified as a consonant, and {\em c} was misclassified as a vowel), and we selected only completely striped or unstriped flags for flag stripe classification.  The SVM-Light \cite{Joachims98} package was used with default parameters and its leave-one-out (LOO) classification results are reported in Figure \ref{fig:results}.

\subsection{Visual Search}
We provide a GUI visual search tool, exemplified in Figure
\ref{fig:tilestree}. Given $n$ images, their embedding into $\reals^d$
and the related probabilistic model for triples, we would like to help
a user find either a particular object she has in mind, or a similar
one. We do this by playing ``20 Questions'' with 8- or 9-tuple queries, generated by an information-gain
adaptive selection algorithm very similar to the one described in Section \ref{sec:ada}.

\smallskip
\noindent
{\bf Acknowledgments.}  We thank Sham Kakade and Varun Kanade for helpful discussions.  Serge Belongie's research is partly funded by ONR MURI Grant N00014-08-1-0638 and NSF Grant AGS-0941760.

\appendix
\section{Proof of Theorem~\ref{thm:one}}
\label{app:proof}
\subsection{Analysis}

Before we present the proof of Theorem~\ref{thm:one}, we introduce a natural generalization which will be a convenient abstraction.  We call this relative regression.

\subsection{Relative regression}

Consider the following online relative regression model.  There is a sequence of examples $(x_1,x_1',y_1),(x_2,x_2',y_2),\ldots,(x_T,x_T',y_T) \in X \times X \times \{0,1\}$, for some set $X \subseteq \reals^d$.  For $w \in \reals^d$, define the relative linear model with $w$ to be,
$$p_t(w) = \frac{w \cdot x_t}{w \cdot x_t + w \cdot x_t'}.$$
The sequence $x_1,x_1',\ldots,x_T,x_T'$ is chosen arbitrary (or even adversarially) in advance; afterwards it is assumed that there is some $w^*\in \reals^d$ such that, $\Pr[y_t=1]=p_t(w^*)$ and that the different $y_t$'s are independent.  It is further assumed that $w^*$ belongs to some convex compact set $W \subset \reals^d$ and that $w \cdot x$ is positive and bounded over $w \in W,x\in X$.  Without loss of generality, by scaling we can require $w\cdot x \in [1,\beta]$ for some $\beta>0$ and every $w \in W,x\in X$.

On the $t$th period, the algorithm outputs $w_t \in W$, then observes $x_t,x_t',y_t$, and finally incurs loss $\ell_t(w_t)$ where $\ell_t(w)=\log 1/p_t(w)$ if $y_t=1$ and $\ell_t(w)=\log 1/(1-p_t(w))$ if $y_t=0$.  The goal of the algorithm is to incur total loss not much larger than $\sum_t \ell_t(w^*)$, the best choice had we known $w^*$ in advance.  

We note that an analogous (and slightly simpler version of) the following lemma can be proven for squared loss.
\begin{lemma}\label{lem:rel}
Let $X,W \subseteq \reals^d$ and suppose that $W$ is compact and convex and $\exists \alpha >0$ such that for all $x \in X$, $w \in W$:
$\|x\|,\|w\|\leq 1$, and $w \cdot x \geq \alpha$.
The for any $\eta>0$ and any $w^0\in W$ and $w^{t+1}=\Pi_W(w^t-\eta \nabla \ell_t(w_t))$,
$$\frac{1}{T}\E\left[\sum_{t=1}^T \ell_t(w_t) -\ell_t(w^*)\right] \leq \frac{\eta}{\alpha^2}+\frac{2}{T\eta \alpha}.$$
\end{lemma}
In particular, for $\eta=\sqrt{2\alpha/T}$, this gives a bound on the right-hand side of $\sqrt{\frac{8}{T\alpha^3}}$.
\begin{proof}
Following the analysis of Zinkevich \cite{Zinkevich03} we consider the potential equal to the squared distance $(w_t-w^*)^2$ and argue that it decreases whenever we have substantial error.
Let $\nabla_t = \nabla \ell_t(w_t)\in \reals^d$, which is,
$$\nabla_t = \frac{x_t+x_t'}{w_t\cdot x_t+w_t\cdot x_t'} - y_t \frac{x_t}{w_t\cdot x_t} - (1-y_t)\frac{x_t'}{w_t\cdot x_t'}.$$
By the triangle inequality $\|\nabla_t\| \leq G$ for $G=\frac{2}{\alpha}$.  Now, as Zinkevich points out, due to convexity of $W$, $(w-\Pi_W(v))^2\leq (w-v)^2$ for any $v \in \reals^d$ and $w \in W$.  Hence,
$$(w^*-w_{t+1})^2 \leq (w^*-w_t+\eta \nabla_t)^2.$$
Thus the {\em decrease} in potential, call it $\Delta_t = (w^*-w_t)^2-(w^*-w_{t+1})^2$, is at least:
\begin{align*}
\Delta_t &\geq (w^*-w_t)^2-(w^*-w_t+\eta \nabla_t)^2 \\
&=2\eta \nabla_t \cdot (w_t-w^*)-\eta^2 \nabla_t^2.
\end{align*}

Next, we consider the quantity,
$\E[ \Delta_t \cdot w^* ]$, where the expectation is taken over the random $y_t$ (fixing $y_1,y_2,\ldots,y_{t-1}$).  By expansion, the expectations is:
$$\frac{w^* \cdot x_t + w^* \cdot x_t'}{w \cdot x_t + w\cdot x_t'} - p_t(w^*) \frac{w^* \cdot x_t}{w_t\cdot x_t} - (1-p_t(w^*))\frac{w^* \cdot x_t'}{w_t\cdot x_t'}.$$
After simple algebraic manipulation, which is difficult to show in two-column format, we have,
\begin{align*}
\E[\Delta_t \cdot w^*]&=-Z_t(p_t(w^*)-p_t(w_t))^2 \text{ where }\\
Z_t &=\frac{(w_t\cdot x_t+w_t\cdot x_t')(w^*\cdot x_t+w^*\cdot x_t')}{(w_t \cdot x_t)(w_t \cdot x_t')}.
\end{align*}

Also note that $\Delta_t \cdot w_t =0$ regardless of $y_t$.  Hence,
$\E[ \Delta_t \cdot w_t] =0$.  Combining these with the fact that we have shown that $\Delta_t \geq 2\eta \nabla_t \cdot(w_t-w^*)-\eta^2 G^2$, gives,

\begin{align*}
\E[\Delta_t] &\geq 2\eta Z_t (p_t(w^*)-p_t(w_t))^2-\eta^2G^2\\
&\geq 2 \eta \frac{w^*\cdot x_t + w^* \cdot x_t'}{w_t\cdot x_t + w_t \cdot x_t'}\frac{(p_t(w_t)-p_t(w^*))^2}{p_t(w_t)(1-p_t(w_t))}-\eta^2G\\
&\geq 2\eta\alpha \frac{(p_t(w_t)-p_t(w^*))^2}{p_t(w_t)(1-p_t(w_t))}-\eta^2G.
\end{align*}
In the last line we have used the fact that $w\cdot x \in [\alpha,1]$ for $w\in W,x\in X$.
Now, by Lemma \ref{lem:appx1} which follows this proof,
$$\ell_t(w_t)-\ell_t(w^*) \leq \frac{(p_t(w_t)-p_t(w^*))^2}{p_t(w_t)(1-p_t(w_t))}.$$
Combining the previous two displayed equations gives,
$$\E[\Delta_t]  \geq 2 \eta\alpha (\ell_t(w_t)-\ell_t(w^*))-\eta^2G.$$
Finally, since the potential $(w_t -w^*)^2>0$, we have $\sum \Delta_t \leq (w_0-w^*)^2 \leq 4.$  Hence,
$$\sum_t \ell_t(w_t)-\ell_t(w^*) \leq \frac{T \eta^2 G+4}{2\eta \alpha}.$$
Substituting $G=2/\alpha$ gives the lemma.
\end{proof}

\begin{lemma}\label{lem:appx1}
Let $p+q=1$ and $p^*+q^*=1$ for $p,p^* \in [0,1]$.  Then,
$$p^* \log \frac{p^*}{p} + q^* \log \frac{q^*}{q} \leq \frac{(p-p^*)^2}{pq}.$$
\end{lemma}
\begin{proof}
By concavity of $\log$, Jensen's inequality implies,
$$p^* \log \frac{p^*}{p} + q^* \log \frac{q^*}{q}  \leq \log \frac{(p^*)^2}{p} + \frac{(q^*)^2}{q}.$$
Simple algebraic manipulation shows that,
$$\frac{(p^*)^2}{p} + \frac{(q^*)^2}{q} = 1 + \frac{(p-p^*)^2}{pq}.$$
Finally, the fact that $\log 1+x \leq x$ completes the lemma.
\end{proof}

\subsection{Proof}

To prove theorem~\ref{thm:one}, map matrix $S\in \reals^{n \times n}$ to a vector $w(S)\in\reals^{1+n^2}$ consisting of the constant $\mu+2$ in the first coordinate followed by the $n^2$ entries of $S$.  Taken over the set of symmetric $S \succeq 0$ such that $S_{ii}=1$, the vectors $w(S)$ for a compact convex set of radius $\sqrt{n^2+(2+\mu)^2}$.  Also,
$$p^{a_t}{b_tc_t}=\frac{\mu + 2-\S_{ac}-\S_{ca}}{2\mu + 4-\S_{ab}-\S_{ba}-\S_{ac}-\S_{ca}},$$
is our relative regression model for $w=w(S)$, $x\in\reals^{1+n^2}$ being the vector with a $1$ is the first position and -1's in the positions corresponding to the $ac$ and $ca$ entries of $S$ (and zero elsewhere), and $x'$ having a 1 in the first position and -1's in the positions corresponding to $ab$ and $ba$ (and zero elsewhere).  The inner product of $w(S)$ and $x$ is $\mu+\delta_{ab}$ and hence is bounded.  
To apply Lemma \ref{lem:rel}, one must scale $w(S)$ and $x,x'$ down.  However, it is clear that for any $\epsilon$ and $T$ sufficiently large, setting $\eta=1/\sqrt{T}$ in Lemma \ref{lem:rel} gives the necessary bound.

\bibliography{sim}
\bibliographystyle{icml2011}

\end{document}